\documentclass[twoside]{article}
\usepackage[accepted]{styles/aistats2014_format/aistats2014}

\usepackage{hyperref}
\usepackage{url}
\usepackage{graphicx}
\usepackage{amsmath, amssymb}
\usepackage{includes/algorithmic}
\usepackage[ruled]{includes/algorithm2e}

\newcommand{\wc}[1]{}
\newcommand{\ct}[1]{}
\newcommand{\trm}[1]{\textit{{#1}}}

\newtheorem{theorem}{Theorem}
\newtheorem{corollary}{Corollary}
\newtheorem{lemma}{Lemma}
\newenvironment{proof}{\textit{Proof.}}{\hfill$\Box$}

\begin{document}
%\maketitle

\twocolumn[
\aistatstitle{Scaling Graph-based Semi Supervised Learning to Large Number of Labels Using Count-Min Sketch}
\aistatsauthor{Partha Pratim Talukdar \\ Machine Learning Department \\ Carnegie Mellon University, USA \\ \texttt{ppt@cs.cmu.edu} \And William Cohen \\ Machine Learning Department \\ Carnegie Mellon University, USA \\ \texttt{wcohen@cs.cmu.edu}}%1 \And Anonymous Author 2 \And Anonymous Author 3 }
%\aistatsaddress{ \AND } %1 \And Unknown Institution 2 \And Unknown Institution 3 }
\aistatsaddress{}
]

\newcommand{\pst}[1]{p_{#1}(+1 | s, t)}
\newcommand{\p}[3]{p_{#1}(+1 | #2, #3)}
\newcommand{\x}[3]{x_{#1}(+1 | #2, #3)}
\newcommand{\pstminus}[1]{p_{#1}(-1 | s, t)}

\newcommand{\refalg}[1]{Algorithm~\ref{#1}}
\newcommand{\refeqn}[1]{Equation~\ref{#1}}
\newcommand{\reffig}[1]{Figure~\ref{#1}}
\newcommand{\reftbl}[1]{Table~\ref{#1}}
\newcommand{\refsec}[1]{Section~\ref{#1}}

\newcommand{\KB}[1]{\textrm{KB}_{#1}}

\newcommand{\query}[1]{\{#1\}}
\newcommand{\sketch}{\mathbb{S}}

\newcommand{\reminder}[1]{}
\newcommand{\method}[1]{\mbox{\textsc{#1}}}
% added negative spacing \! 
\newcommand{\norm}[1]{\left|\!\left|#1\right|\!\right|^{2}_2}
% wc addition
\newcommand{\onenorm}[1]{|\!|#1|\!|_1}
\newcommand{\boundary}{\partial}
\newcommand{\vol}{\textit{vol}}

\newcommand{\tinBE}{$t \in \{ B, \ldots, E \}$}
\newcommand{\xrst}{x_{r,s,t}}
\newcommand{\zbrst}{z^{b}_{r,s,t}}
\newcommand{\zerst}{z^{e}_{r,s,t}}
\newcommand{\qrst}{\phi_{r}(s,t,+1)}
\newcommand{\qrt}[1]{\phi_{r}(#1,t,+1)}
\newcommand{\qrsty}{\phi_{r}(s,t,y)}
\newcommand{\qrty}[1]{\phi_{r}(#1,t,y)}
\newcommand{\bw}{{\bf w}}

%\newcommand{\YH}{\hat{Y}}
% \hat{Y} => Y
\newcommand{\YH}{Y}

%\newcommand{\YTW}{\tilde{Y}}
% \tilde{Y} => \hat{Y}
\newcommand{\YTW}{\hat{Y}}

\newcommand{\SH}{\hat{S}}

%\newcommand{\YBl}[0]{\mathbf{Y}_{l}}
% Y => Q
\newcommand{\YBl}[0]{\mathbf{Q}_{l}}

\newcommand{\YB}[1]{\mathbf{Q}_{#1}}

\newcommand{\YHBl}[0]{\hat{\mathbf{Q}}_{l}}

\newcommand{\MB}[0]{\mathbf{M}}
\newcommand{\WB}[0]{\mathbf{W}}
\newcommand{\RB}[0]{\mathbf{R}}

\newcommand{\pucont}[0]{p_u^{\textrm cont}}

\newcommand{\reals}{\mathbb{R}}
\newcommand{\realsp}{\reals_+}

\newcommand{\eat}[1]{}
\newcommand{\notforproposal}[1]{}

\newcommand{\VT}[0]{\tilde{V}}
\newcommand{\vT}[0]{\tilde{v}}
\newcommand{\VP}[0]{V^{'}}
\newcommand{\Vs}[0]{V_{s}}

\newcommand{\EP}[0]{E^{'}}

\newcommand{\pvcont}[0]{p_v^{cont}}
\newcommand{\pvinj}[0]{p_v^{inj}}
\newcommand{\pvabnd}[0]{p_v^{abnd}}

\newcommand{\MadExact}{\method{MAD-Exact}}
\newcommand{\MadSketch}{\method{MAD-Sketch}}

\begin{abstract}
Graph-based Semi-supervised learning (SSL) algorithms have been successfully used in a large number of applications. These methods classify initially unlabeled nodes by propagating label information over the structure of graph starting from seed nodes. Graph-based SSL algorithms usually scale linearly with %the number of edges ($|E|$) and also in 
the number of distinct labels ($m$), and require $O(m)$ space on each node. Unfortunately, there exist many applications of practical significance with very large $m$ over large graphs, demanding better space and time complexity. In this paper, we propose \MadSketch{}, a novel graph-based  SSL algorithm which compactly stores label distribution on each node using Count-min Sketch, a randomized data structure. We present theoretical analysis showing that under mild conditions, \MadSketch{} can reduce space complexity at each node from $O(m)$ to $O(\log m)$, and achieve similar savings in time complexity as well. We support our analysis through experiments on multiple real world datasets. We observe that \MadSketch{} achieves similar performance as existing state-of-the-art graph-based SSL algorithms, while requiring smaller memory footprint and at the same time achieving up to 10x speedup. We find that \MadSketch{} is able to scale to datasets with one million labels, which is beyond the scope of existing graph-based SSL algorithms.
\end{abstract}

\section{Introduction}

Graph-based semi-supervised learning (SSL) methods based on label
propagation work by associating classes with each labeled ``seed''
node, and then propagating these labels over the graph in some
principled, iterative manner \cite{zhu2003semi,talukdar:nra09}.  The
end result of SSL is thus a matrix $Y$, where $Y_{u\ell}$ indicates
whether node $u$ is positive for class $\ell$.  SSL techniques based
on label propagation through graphs are widely used, and they are
often quite effective for very large datasets.  Since in many cases
these method converge quite quickly, the time and space requirements
scale linearly with both the number of edges $|E|$ in the graph, and
the number of classes $m$.

Unfortunately, there exist a number of applications where both $m$ and
$|E|$ are very large: for instance, Carlson et al. 
\cite{DBLP:conf/aaai/CarlsonBKSHM10} describe an SSL system with
hundreds of overlapping classes, and Shi et al \cite{shi2009hash}
describe a text classification task with over 7000 classes. Similarly, the ImageNet dataset \cite{deng2009imagenet} poses a classification task involving 100,000 classes. In this paper, we seek to extend graph-based SSL methods to cases where there are a large number of potentially overlapping labels.  To do this, we represent
the class scores for $u$ with a \trm{count-min sketch} \cite{cormode2005improved}, a randomized
data structure.

Graph-based SSL using a count-min sketch has a number of properties
that are desirable, and somewhat surprising.  First, the sketches can
be asymptotically smaller than conventional data structures.  Normally,
the vectors of label scores that are propagated through a graph are
dense, i.e. of size $O(m)$.  However, we show that if the seed labels
are sparse, or if the label scores at each node have a ``power-law''
structure, then the storage at each node can be reduced from $O(m)$ to
$O(\log m)$.  Experimentally, we show that power-law label scores
occur in natural datasets.  Analytically, we also show that a similar
space reduction can be obtained for graphs that have a certain
type of community structure.  A second useful property of the
count-min sketch is that for label propagation algorithms using a
certain family of updates---including Modified Adsorption (MAD)
\cite{talukdar:nra09}, the harmonic method \cite{zhu2003semi} and
several other methods \cite{wu2012learning}---the linearity of the
count-min sketch implies that a similarly large reduction in
\emph{time complexity} can be obtained.  Experimentally, little or no
loss in accuracy is observed on moderate-sized datasets, and \MadSketch{}, the new method scales to a SSL problem with millions of edges and nodes, and one million labels.

\label{sec:intro}

\section{Related Work}
\label{sec:related}

Although tasks with many labels are common, there has been
surprisingly little work on scaling learning algorithms to many
labels. Weinberger et al. \cite{Weinberger:2009:FHL:1553374.1553516}
describe a feature-hashing method called a ``hash kernel'', which is
closely related to a count-min sketch, and show that it can help
supervised learning methods scale to many classes.  Other approaches
to scaling classifiers to many classes include reducing the dimension
of label space, for instance by error-correcting output codes
\cite{dekel2002multiclass} or label embedding trees
\cite{bengio2010label}.  Certain supervised learning methods, notably
nearest-neighbor methods, also naturally scale well to many classes
\cite{YangChute94}.  The main contribution of this work is provide
similarly scalable methods for SSL methods based on label propagation. 
A graph-based SSL method aimed at retaining only top ranking labels 
per-node out of a large label set is presented in \cite{agrawal2013multi}. 
Similar ideas to induce sparsity on the label scores were also explored in \cite{talukdar2010experiments,das2012graph}. In contrast to these methods, the method 
presented in paper doesn't attempt to enforce sparsity, and instead 
focuses on compactly storing the \emph{entire} label distribution on each node 
 using Count-min Sketch.

The count-min sketch structure has been used for a number of tasks,
such as for the storage distributional similarity data for words in
NLP \cite{goyal2010sketch}.  Perhaps the most related of these tasks
is the task of computing and storing personalized PageRank vectors for
each node of a graph \cite{sarlos2006randomize,song2009scalable}.
Prior analyses, however, do not apply immediately to SSL, although
they can be viewed as bounds for a special case of SSL in which each
node is associated with a 1-sparse label vector and there is one label
for every node.  We also include new analyses that provide tighter
bounds for skewed labels, and graphs with community structure.

\section{Preliminaries}
\label{sec:prelim}

\subsection{Count-Min Sketches (CMS)}
\label{sec:cms_review}

The count-min sketch is a widely-used probabilistic scheme \cite{cormode2005improved}. At a high level, it stores an approximation of a mapping between integers $i$ and associated real values $y_i$.  Specifically, the count-min sketch consists of a $w \times d$ matrix $\sketch{}$, together with $d$ hash functions $h_1,\ldots,h_d$, which are chosen randomly from a pairwise-independent family. A sketch is always initialized to an all-zero matrix.

Let $\textbf{y}$ be a sparse vector, in which the $i$-th component has
value $y_i$.  To store a new value $y_i$ to be associated with
component $i$, one simply updates as follows: %increments $\sketch{}_{j,h_j(i)}$ by $y_i$ for every $j:1\leq{}j\leq{}d$: i.e., for each row $j$ of $\sketch{}$, one increments the bucket $h_j(i)$ that $i$ hashes to with the $j$-th hash function.  
\[
	\sketch{}_{j,h_j(i)} \leftarrow \sketch{}_{j,h_j(i)} + y_i,~\forall 1 \leq{} j \leq{} d
\]
To retrieve an approximate value $\hat{y}_i$ for $i$, one computes the minimum value over all $j:1\leq{}j\leq{}d$ of $\sketch{}_{j,h_j(i)}$ as follows:
\begin{equation}
	\hat{y_i} = \min_{1 \leq{} j \leq{} d} \sketch{}_{j, h_j(i)} \label{eqn:cms_query}
\end{equation}
We note that $\hat{y_i}$ will never underestimate $y_i$, but may overestimate it if there are hash collisions.  As it turns out, the parameters $w$ and $d$ can be chosen
so that with high probability, the overestimation error is small (as
we discuss below, in Section~\ref{sec:analysis}).

Importantly, count-min sketches are linear, in the following sense: if
$\sketch{}_1$ is the matrix for the count-min sketch of $\textbf{y}_1$ and
$\sketch{}_2$ is the matrix for the count-min sketch of $\textbf{y}_2$, then
$a \sketch{}_1 + b \sketch{}_2$ is the matrix for the count-min sketch of the vector $a \textbf{y}_1 + b \textbf{y}_2$.

\subsection{Graph-based Semi-Supervised Learning}
\label{sec:graph_ssl_review}

%Following \cite{talukdar2010experiments}, 
We now briefly review graph-based SSL algorithms. 

\subsection*{Notation}
\label{sec:notation}

%Let $X$ be the $d \times n$ matrix of $n$ instances in a $d$-dimensional space. Out of the $n$ instances, $n_l$ instances are labeled, while the remaining $n_u$ instances are unlabeled, with $n = n_l + n_u$. 
All the algorithms compute a soft assignment of labels to the nodes of
a graph $G = (V, E, W)$, where $V$ is the set of nodes with $|V| = n$, $E$ is the set of edges, and $W$ is an edge weight matrix. Out of the $n = n_l + n_u$ nodes in $G$, $n_l$ nodes are labeled, while the remaining $n_u$ nodes are unlabeled. %, and $W$ is the symmetric $n \times n$ matrix of edge weights that we would like to learn. $W_{ij}$ is the  weight of edge $(i, j)$, and also the similarity between instances $x_i$ and $x_j$.
If edge $(u,v)\not\in E$, $W_{uv}=0$. The (unnormalized) Laplacian,
$L$, of $G$ is given by $L = D - W$, where $D$ is an $n \times n$
diagonal degree matrix with $D_{uu} = \sum_v W_{uv}$. Let $S$ be an $n
\times n$ diagonal matrix with $S_{uu} = 1$ iff node $u \in V$ is
labeled. That is, $S$ identifies the labeled nodes in the graph. 
$C$ is the set of labels, with $|C| = m$ representing the total number
of labels. $Q$ is the $n \times m$ matrix storing training label
information, if any.  $\YH$ is an $n \times m$ matrix of soft label assignments, with $\YH_{vl}$
representing the score of label $l$ on node $v$. A graph-based
SSL computes $\YH$ from $\{ G, S, Q \}$.

\subsection*{Modified Adsorption (MAD)}
\label{sec:mad}

%\begin{algorithm}[t]
%  \caption{MAD Algorithm
%      \label{alg:modi_adsp}}
%    {\bf Input}: \\
%    \begin{tabular}{lll}
%      - & {\bf Graph:} & $G = (V,E, W)$\\
%      - & {\bf Prior labeling:} & $\YB{v}\in\reals^{m+1}$ for $v \in V$\\
%      - & {\bf Probabilities:} & $\pvinj, \pvcont, \pvabnd$ for $v \in V$
%\end{tabular}
%
%    {\bf Output}: \\
%    \begin{tabular}{lll}
%      - & {\bf Label Scores:}  & $\YHB{v}$ for $v \in V$
%\end{tabular}
%  \begin{algorithmic}[1]
%   \STATE $\YHB{v} \leftarrow \YB{v}$ for $v \in V$
%   \COMMENT{Initialization}
%\STATE
%\(
%  \MB_{vv} \leftarrow\mu_1 \times \pvinj + \mu_2 \times \pvcont \times \sum_{u}
%  \WB_{vu} + \mu_3 \)
%%    \STATE
%    \REPEAT
%%    \STATE $N_v \leftarrow \sum_{v'}W(v',v)$
%    \STATE \(
%    D_v \leftarrow 
%\displaystyle
%\sum_{u} \left(\pvcont \WB_{vu} + \pucont \WB_{uv}\right) \YHB{u}
%\)
%    \FORALL{$v \in V$}
%    \STATE   $\YHB{v} \leftarrow \frac{1}{\MB_{vv}} \left(
%    \mu_1 \times \pvinj\times\YB{v} +
%  \mu_2  \times D_v
% + \mu_3 \times\pvabnd \times\RB_{v} \right)$
%    \ENDFOR
%    \UNTIL{convergence}
%  \end{algorithmic}
%\end{algorithm}

Modified Adsorption (MAD) \cite{talukdar:nra09} is a modification of
an SSL method called Adsorption \cite{baluja2008video}. MAD can be expressed as an
unconstrained optimization problem, whose objective is shown below.
\begin{eqnarray}
\centering
  \min_{\YH} && \sum_{l \in C} \left[ \mu_1 
\left( Q_{l} - \YH_{l}\right)^\top S \left( Q_{l} - \YH_{l} \right) + \right. \nonumber \\
&& \quad \quad \left. \mu_2~\YH_l^{\top}~L^{'}~\YH_l~+~\mu_3 \norm{\YH_l - R_l} \right] \label{eqn:mad_obj}
\end{eqnarray}

\noindent where $\mu_1$, $\mu_2$, and $\mu_3$ are hyperparameters;
$L^{'} = D^{'} - W^{'}$ is the Laplacian of an undirected graph derived from $G$, but
with revised edge weights; and $R$ is an $n \times m$ matrix of
per-node label prior, if any, with $R_{l}$ representing the $l^{\mathrm{th}}$ column of $R$. As in Adsorption, MAD allows labels on seed nodes to change. In case of MAD, the three random-walk probabilities, $\pvinj,~\pvcont,~\mathrm{and}~\pvabnd$, defined by Adsorption on each node are folded inside the matrices $S, L^{'}, \mathrm{and}~R$, respectively. The optimization problem in (\ref{eqn:mad_obj}) can be solved
with an efficient iterative update which is shown below and repeated until convergence, described in detail by~\cite{talukdar:nra09}.
\begin{equation} \label{eq:update}
\begin{array}{c}
	\YH{}^{(t + 1)}_{v} \leftarrow \frac{1}{M_{vv}} (\mu_1 \times S_{vv} \times Q_v + \mu_2 \times D_{v}^{(t)} \nonumber \\
	 \hspace{2cm} +~\mu_3 \times R_v),~\forall v \in V \\
	\mathrm{where}~~D_{v}^{(t)} = \sum_{u \in V} (W_{uv}^{'} + W_{vu}^{'}) \times \YH{}_{u}^{(t)},\\
	\mathrm{and}~M_{vv} = \mu_1 S_{vv} + \mu_2 \sum_{u \ne v} (W_{uv}^{'} + W_{vu}^{'}) + \mu_3
\end{array}
\end{equation}
%
%
%These three algorithms are all easily parallelizable in a MapReduce
%framework~\cite{talukdar2008wsa,rao2009ranking}, which makes them
%suitable for SSL on large datasets. Additionally, all three algorithms have similar space and time complexity.
Typically $\YH^{(t)}$ is sparse for very small $t$, and becomes dense in later iterations.
MAD is easily parallelizable in a MapReduce framework~\cite{talukdar2008wsa,rao2009ranking}, which makes it suitable for SSL on large datasets.

\section{\MadSketch{}: Using Count-Min Sketch for Graph SSL}
\label{sec:method}

%\wc{in analysis section, I use $\YTW$ for the output of the count-min sketch method which approximates $\YH$, and $W'$ for the revised weight matrix}
%
In this section, we propose \MadSketch{}, a count-min sketch-based extension of MAD\footnote{We use MAD as a representative graph-based SSL algorithm, which also generalizes prior work of \cite{zhu2003semi}.} \cite{talukdar:nra09}. Instead of storing labels and their scores exactly as in MAD, \MadSketch{} uses count-min sketches to compactly store labels and their scores. \MadSketch{} uses the following equation to update label sketches at each node, and this process is repeated until convergence.
\begin{eqnarray}
	\sketch{}_{Y,v}^{(t + 1)} &\leftarrow& \frac{1}{M_{vv}} \left( \mu_1 \times S_{vv} \times \sketch{}_{Q,v} + \right. \nonumber \\
	&& \mu_2 \times \sum_{u \in V} (W_{uv}^{'} + W_{vu}^{'}) \times \sketch{}_{Y,u}^{(t)} + \nonumber \\
	&& \left. \mu_3 \times \sketch{}_{R,v} \right),~\forall v \in V
\end{eqnarray}
where $\sketch{}_{Y,v}^{(t + 1)}$ is the count-min sketch corresponding to label score estimates on node $v$ at time $t + 1$, $\sketch{}_{Q,v}$ is the sketch corresponding to any seed label information, and finally $\sketch{}_{R,v}$ is the sketch corresponding to label regularization targets in node $v$. Please note that due to linearity of CMS, we don't need to unpack the labels for each update operation. This results in significant runtime improvements compared to non-sketch-based MAD as we will see in \refsec{sec:expt_performance}.

After convergence, \MadSketch{} returns $\sketch{}_{Y,v}$ as the
count-min sketch containing estimated label scores on node $v$. The
final label score of a label can be obtained by querying this sketch
using \refeqn{eqn:cms_query}.  We denote the result of this query by
$\YTW$.

\section{Analysis}
\label{sec:analysis}

\subsection{Sparse initial labelings}

We now turn to the question of how well \MadSketch{} approximates the
exact version---i.e., how well $\YTW$ approximates $\YH$.  We begin
with a basic result on the accuracy of count-min sketches
\cite{cormode2005improved}.

\begin{theorem}[Cormode and Muthukrishnan \cite{cormode2005improved}] \label{thm:cm1}
Let $\textbf{y}$ be an vector, and let $\tilde{y}_i$ be the estimate
given by a count-min sketch of width $w$ and depth $d$ for $y_i$.  If
$w\geq{}\frac{e}{\eta}$ and $d\geq{}\ln(\frac{1}{\delta})$, then \(
\tilde{y_i} \leq y_i + \eta \onenorm{\textbf{y}} \) with probability
at least 1-$\delta$.
\end{theorem}

To apply this result, we consider several assumptions that might be
placed on the learning task.  One natural assumption is that the
initial seed labeling is \trm{$k$-sparse}, which we define to be a
labeling so that for all $v$ and $\ell$, $\onenorm{Y_{v\cdot}} \leq k$
for some small $k$ (where $Y_{v\cdot}$ is the $v$-th row of $Y$).  If
$\YTW$ is a count-min approximation of $\YH$, then we define the
\trm{approximation error of $\YTW$ relative to $\YH$} as \(
\max_{v,\ell} (\YTW_{v,\ell} - \YH_{v,\ell}) \). Note that
approximation error cannot be negative, since a count-min sketch can
only overestimate a quantity.

\begin{theorem}  \label{thm:main}
If the parameters of \MadSketch{} $\mu_1,\mu_2,\mu_3$ are such that
$\mu_1 + \mu_2 + \mu_3 \leq 1$, $Y$ is $k$-sparse and binary, and
  sketches are of size $w \geq \frac{ek}{\epsilon}$ and $d \geq
  \ln\frac{m}{\delta}$ then the approximation error \MadSketch{} is
  less than $\epsilon$ with probability 1-$\delta$.
\end{theorem}

\begin{proof}
The result holds immediately if $\forall v$, $\onenorm{\YH_{v\cdot}}
\leq k$, by application of the union bound and Theorem~\ref{thm:cm1}
with $\eta=\epsilon/k$.  However, if $\mu_1 + \mu_2 + \mu_3 \leq 1$,
then at each iteration of MAD's update, Equation~\ref{eq:update}, the
new count-min sketch at a node is bounded by weighted average of
previously-computed count-min sketches, so by induction the bound on
$\onenorm{\YH_{v\cdot}}$ will hold.
\end{proof}

Although we state this result for \MadSketch{}, it also holds for
other label propagation algorithms that update scores by a weighted
average of their neighbor's scores, such as the harmonic method
\cite{zhu2003semi}.

\subsection{Skewed labels}

Perhaps more interestingly, the label weights $\YH_{v\cdot}$
associated with real datasets exhibit even more structure than
sparsity alone would suggest.  Define a Zipf-like distribution to be
one where the frequency of the $i$-th most frequent value
%$f_i={C}{i^{-z}}$, where $C$ is a normalizing constant
$f_i \propto i^{-z}$.
For graphs from two real-world datasets, one with 192 distinct labels and another with 10,000 
labels, Figure~\ref{fig:freebase_lab_score_dist} shows the label scores at
each rank, aggregated across all nodes in these two graphs.  Note that the score distributions are
Zipf-like\footnote{Although in the case of the 192-label dataset, the
  Zipf-like distribution is clearly truncated somewhere below
  $i=192$.} with $z\approx{}1$, indicating a strong skew toward a few
large values.

This structure is relevant because it is known that count-min sketches
can store the largest values in a skewed data stream even more
compactly.  For instance, the following lemma can be extracted from
prior analysis \cite{cormode2005summarizing}.

\begin{lemma}{\small {\bf [Cormode and Muthukrishnan \cite{cormode2005summarizing}, Eqn~5.1]}} \label{lemma:cm2}
Let $\textbf{y}$ be an vector, and let $\tilde{y}_i$ be the estimate
given by a count-min sketch of width $w$ and depth $d$ for $y_i$.  Let
the $k$ largest components of $\textbf{y}$ be
$y_{\sigma_1},\ldots,y_{\sigma_k}$, and let $t_k=\sum_{k'>k}
y_{\sigma_{k^{'}}}$ be the weight of the ``tail'' of $\textbf{y}$.  If
$w\geq{}\frac{1}{3k}$, $w>\frac{e}{\eta}$ and $d\geq{} \ln\frac{3}{2}
\ln\frac{1}{\delta}$, then \( \tilde{y_i} \leq y_i + \eta t_k \) with
probability at least 1-$\delta$.
\end{lemma}

The proof for this statement comes from breaking down the count-min
error analysis into two independent parts: (1) errors due to
collisions among the $k$ largest values, which are unlikely since
there are only $k$ values to be hashed into the $w>3k$ bins; and (2)
errors due to collisions with the tail items, which have small impact
since $t_k$ is small.  As a application of this analysis, Cormode and
Muthukrishnan also showed the following.
\begin{theorem}{\small {\bf [Cormode and Muthukrishnan \cite{cormode2005summarizing}, Theorem 5.1]}} \label{thm:cm3}
Let $\textbf{y}$ represent a Zipf-like distribution with parameter
$z$.  Then with probability at least 1-$\delta$, $\textbf{y}$ can be
approximated to within error $\eta$ by a count-min sketch of width
$O(\eta^{-\min(1,1/z)})$ and depth $O(\ln\frac{1}{\delta})$.
\end{theorem}
Note that this result holds even the absence of label sparsity; also,
it gives a strictly more space-efficient sketch when $z>1$ (for
instance, when $z=1.5$, then we can reduce the width of each sketch to
$\lceil\frac{e}{\sqrt{\epsilon}}\rceil$), as observed below.

\begin{corollary} If the vector label scores $\YH_{v\cdot}$ for 
every node $v$ is bounded by a Zipf-like distribution parameter $z>1$,
and sketches are of size $w \geq \frac{e}{\epsilon^{z-1}}$ and $d \geq
\ln\frac{m}{\delta}$ then the approximation error of \MadSketch{} is
less than $\epsilon$ with probability 1-$\delta$.
\end{corollary}

Again, although we state this result for \MadSketch{}, it would also
hold for sketch versions of other label propagation algorithms: the
only requirement is that label scores are (empirically) always skewed
during propagation.

\subsection{Graphs with community structure}

Lemma~\ref{lemma:cm2} can also be used to derive a tighter bound for a
certain type of nearly block-diagonal graph.  Let $G=(V,E,W)$ be a
graph, let $S \subset{}V$ be a set of vertices.  For a vertex
$u\in{}S$ we define $\boundary(u,S)$ to be the total weight
of all edges from $u$ to a vertex outside of $S$, i.e. 
\( \boundary(u,S) \equiv \sum_{v\not\in{}S} W_{uv}
\). 
For convenience, we define $\vol(u) \equiv D_{uu}$, and $\vol(S) \equiv \sum_{u\in{}S} \vol(u)$.  
The \trm{max-conductance of $S$}, denoted $\psi(S)$, is
defined to be
\[ \psi(S) \equiv \max_{u\in{}S} \frac{\boundary(u,S)}{\vol(u)}
\]
Notice that this is different from the conductance of $S$, which is
often defined as
\[ \phi(S) \equiv \frac{\sum_{u\in{}S} \boundary(u,S) }{\min(\vol(S), \vol(V) - \vol(S))}
\]
When $\vol(S)$ is small and edge weights are uniform, conductance
$\phi(S)$ can be thought of as the probability that an edge will leave
$S$ when it is chosen uniformly from the set of all edges exiting any
node $u\in{}S$.  In contrast, the quantity
$\frac{\boundary(u,S)}{\vol(u)}$ is the probability that an edge will
leave $S$ when it is chosen uniformly from the set of all edges
exiting a \trm{particular} node $u$, and max-conductance is the
maximum of this quantity over $u\in{}S$.  It is not hard to see that
$\psi(S) \geq \phi(S)$ for any $S$, so having small max-conductance is
a stronger condition than having small conductance.  We have the
following result.

\begin{theorem}
Let $u$ be a vertex and $S$ be a vertex set such that $u\in{}S$,
$|S|=k$, and the min-conductance of $S$ is $\psi$, as measured using
the modified weight matrix $W'$ used by MAD.  Assume also that initial
labels are binary, and that $w\geq{}\frac{1}{3k}$, $w>\frac{e}{\eta}$
and $d\geq{} \ln\frac{3}{2} \ln\frac{1}{\delta}$.  Then for all
$\ell$, the approximation error of \MadSketch{} is bounded by $\eta
\psi$.
\end{theorem}

\newcommand{\walk}{\stackrel{W'}{\rightarrow}}
\newcommand{\swalk}{\stackrel{S,W'}{\rightarrow}}

\begin{proof} Let $\Pr(u\walk{}v)$ be the probability of a random
walk from $u$ to $v$ in the transition matrix defined by $W'$, and let
$\Pr(u\swalk{}v)$ be the (improper) probability\footnote{An improper
  distributions may sum to less than one.} of a random walk restricted
to vertices in $S$.

We observe that min-conductance of $S$ can be used to bound
$\sum_{v\not\in{}S} \Pr(u\walk{}v)$.  Let $B_S$ be the set of vertexes
not in $S$, but connected to some node $z\in{}S$.  Since any path to
$v\not\in{}S$ must pass through some $z\in{}B_S$, we have that
\begin{eqnarray*}
\sum_{v\not\in{}S} \Pr(u\walk{}v) & \leq & \sum_{b\in{}B_S} \Pr(u\walk{}b) \\ 
                                &=& \sum_{a\in{}S,b\in{}B_S} \Pr(u\swalk{}a) W'_{ab} \\
				&=& \sum_{a\in{}S} \Pr(u\swalk{}a) \sum_{b\in{}B_S} \frac{W_{ab}}{\vol(a)} \\
				&\leq& \sum_{a\in{}S} \Pr(u\swalk{}a) \psi \\
				&\leq& \psi \\
\end{eqnarray*}
It can be shown that $\YH_{u\ell} \leq \sum_{v:Y_{v\ell}=1}
\Pr(u\walk{}v)$---i.e., that MAD is based on a partially absorbing
random walk \cite{wu2012learning}; hence the theorem holds by
application of Lemma~\ref{lemma:cm2} with the tail $t_k$ defined as
the weight of $\sum_{v\not\in{}S} \Pr(u\walk{}v)$.
\end{proof}

An immediate consequence of this is that if $W'$ is composed of
size-$k$ subcommunities with min-conductance $\psi$, then the sketch
size can again be reduced by a factor of $\psi$:

\begin{corollary}  If $Y$ is binary and $W'$ has the property that 
all nodes $u$ are inside some subset $S$ such that $|S|<k$ and
$\psi(S)\leq{}\psi$, and if sketches are of size $w \geq
\frac{e\psi}{\epsilon}$ and $d \geq \ln\frac{m}{\delta}$ then the
approximation error of \MadSketch{} is less than $\epsilon$ with
probability 1-$\delta$.
\end{corollary}

Again, although we state this result for MAD, it would also hold for
sketch versions of other label propagation algorithms that can be modeled as
partially absorbing random walks, which include the harmonic functions
method and several other well-known approaches to SSL
\cite{wu2012learning}.

\section{Experiments}
\label{sec:expt}

%\wc{maybe a little more about the flikr experiments?}

\begin{table*}[t]
\centering
\begin{tabular}{|l|c|c|c|c|c|c|c|}
\hline
Name & Nodes ($n$) & Edges & Labels (m) & Seed Nodes & $k-$Sparsity & $\lceil \frac{e k}{\epsilon} \rceil $& $\lceil \ln \frac{m}{\delta} \rceil$ \\
\hline
Freebase &  301,638 & 1,155,001 & 192 & 1917 & 2 & 109 & 8 \\
\hline
Flickr-10k & 41,036 &   73,191 & 10,000 & 10,000 & 1 & 55 & 12 \\
\hline
Flickr-1m & 1,281,887 & 7,545,451 & 1,000,000 & 1,000,000 & 1 & 55 & 17 \\
\hline
\end{tabular}
\caption{\label{tbl:datasets}Description of various graph-based datasets used in \refsec{sec:expt} using $\epsilon = 0.05, \delta = 0.1$. Please note that Flickr-10k is a strict subset of the Flickr-1m dataset. The last two columns compute the width and depth of a Count-min Sketch as prescribed by Theorem 2. See \refsec{sec:expt_setup} for further details.}
\end{table*}

\begin{figure}[t]
\begin{minipage}[b]{0.45\linewidth}
\begin{center}
\includegraphics[scale=0.3,keepaspectratio=true]{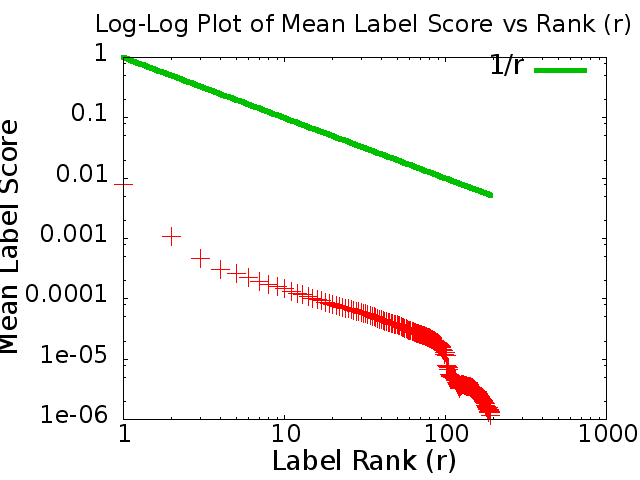}
(a) Freebase
\end{center}
\end{minipage}
\hspace{10mm}
\begin{minipage}[b]{0.45\linewidth}
\begin{center}
\includegraphics[scale=0.3,keepaspectratio=true]{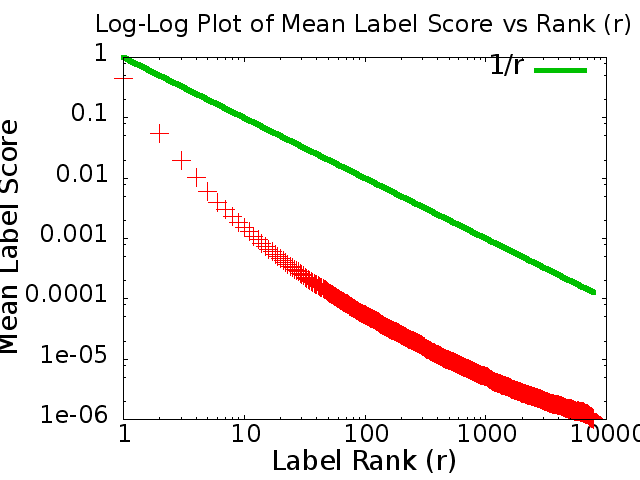}
(b) Flickr-10k
\end{center}
\end{minipage}
\caption{\label{fig:freebase_lab_score_dist}Plots in log-log space demonstrating skewness of label scores estimated by \MadExact{} over two datasets: (a) Freebase and (b) Flickr-10k. For reference, the plot for $\frac{1}{r}$ is also shown, which is Zipfian law with $z = 1$. Please note that the label scores used here are exact, and not approximated using any sketch. See \refsec{sec:label_skew} for details.}
\end{figure}

\begin{table*}[t]
\centering
%\wc{why is storage for w=20, d=3 more than w=20, d=8?}
\begin{tabular}{|l|c|c|c|}
\hline
& Average Memory & Total Runtime (s) & MRR \\
& Usage (GB) & [Speedup w.r.t. {\small \MadExact{}}] & \\
\hline
\MadExact{} & 3.54 & 516.63 [1.0] & 0.28 \\
\MadSketch{} ($w = 109, d = 8$) & 2.68 & 110.42 [4.7] & 0.28 \\
\hline
\hline
\MadSketch{} ($w = 109, d = 3$) & 1.37 & 54.45 [9.5] & 0.29 \\
\MadSketch{} ($w = 20, d = 8$) & 1.06 & 47.72 [10.8] & 0.28 \\
\MadSketch{} ($w = 20, d = 3$) & 1.12 & 48.03 [10.8] & 0.23 \\
\hline
\end{tabular}
\caption{\label{tbl:mem_time_mrr_compare}Table comparing average per-iteration memory usage (in GB), runtime (in seconds), and MRR of \MadExact{} and \MadSketch{} (for various sketch sizes) when applied over the Freebase dataset. Using sketch size of $w = 109$ and $d = 8$ as prescribed by Theorem 2  for this dataset, we find that \MadSketch{} is able to obtain the same MRR (performance) as \MadExact{}, while using a reduced memory footprint and achieving about $4.7$ speedup. This is our main result in the paper. Additionally, we find that even more aggressive sketches (e.g., $w = 20, d = 8$ in this case) may be used in practice, with the possibility of achieving further gains in memory usage and runtime. See \refsec{sec:expt_performance} for details.}
\end{table*}

%\begin{figure}[t]
%\centering
%\includegraphics[scale=0.25,keepaspectratio=true]{figs/mad_freebase_memory_usage_plot.pdf}
%\caption{\label{fig:mad_fbase_mem_usage}Total per-iteration memory usage by MAD when the labels are stored as Strings, Integers, and Count-Min Sketch per node in the Freebase dataset.}
%\end{figure}

\begin{figure*}[t]
\centering
\includegraphics[scale=0.4,keepaspectratio=true]{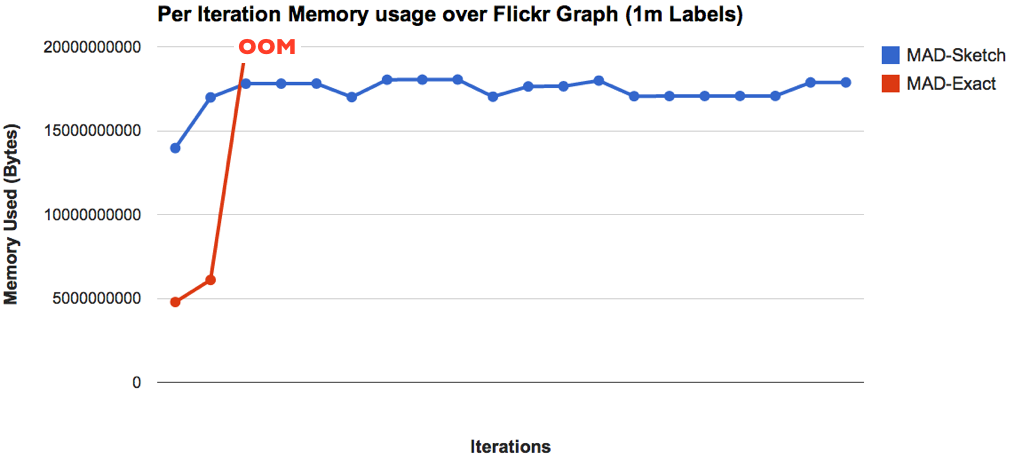}
%\caption{\label{fig:mad_flickr_mem_usage}Total per-iteration memory usage by MAD when the labels are stored as Strings, Integers, and Count-Min Sketch per node in the Flickr dataset.}
\caption{\label{fig:mad_flickr_mem_usage}Per-iteration memory usage by MAD when labels and their scores on each node are stored exactly (MAD-Exact) vs using Count-Min Sketch (MAD-Sketch) in the Flickr-1m dataset, with total unique labels $m = 1000000$, and sketch parameters $w = 55$, $d = 17$. We observe that even though MAD-Exact starts out with a lower memory footprint, it runs out of memory (OOM) by third iteration, while MAD-Sketch is able to compactly store all the labels and their scores with almost constant memory usage over the 20 iterations. See \refsec{sec:large_labels} for details.}
\end{figure*}

\subsection{Experimental Setup}
\label{sec:expt_setup}

The three real-world datasets used for the experiments in this section
are described in \reftbl{tbl:datasets}. The Freebase dataset was
constructed out of 18 domains (e.g., \textit{music, people} etc.) from
the full dataset, and was previously used in
\cite{talukdar2010experiments} for determining semantic types of
entities. This dataset consists of 192 labels, with 10 seeds per
labels. This results in the seeding of 1917 nodes, with three nodes
acquiring 2 labels each. This results in a $k-$sparsity with $k =
2$. Both Flickr-10k and Flickr-1m datasets consist of edges connecting
images to various metadata, e.g., tags, geo-location, etc. For these
two datasets, each seed node was injected with its own unique label,
and hence the number of labels are equal to the number of seed
nodes. Given seed node, the goal here is to identify other similar nodes (e.g., other images, tags, etc.). This self-injection-based seeding scheme results in a $k-$sparsity with $k = 1$. In the last two columns of \reftbl{tbl:datasets}, we calculate the width and depth of sketches as prescribed by Theorem 2 for $\epsilon = 0.05$ and $\delta = 0.1$.

We shall compare two graph-based SSL algorithms: {\bf \MadExact{}},
the default MAD algorithm \cite{talukdar:nra09}, where labels and
their scores are stored exactly on each node without any
approximation, and {\bf \MadSketch{}}, the sketch-based version of
MAD, where the labels and their scores on each node are stored inside
a count-min sketch.\footnote{We use the implementation of MAD in the
  Junto toolkit (https://github.com/parthatalukdar/junto) %https://code.google.com/p/junto/) 
  as \MadExact{}, and
  extend it to handle sketches resulting in the implementation of
  \MadSketch{}. We model our count-min sketches after the ones in the
  springlib library (https://github.com/clearspring/stream-lib), and
  use linear hash functions \cite{carter1979universal}. \MadSketch{} source
  code is available as part of the Junto toolkit or by contacting the authors.}. For the experiments in 
  this paper, we found about 10 iterations to be sufficient for convergence of all 
  algorithms.

Once we have label assignments $\YH{}$ on the evaluation nodes on the
graph, we shall use Mean Reciprocal Rank (MRR) as the evaluation
metric, where higher is better.\footnote{ MRR is defined as \(
  \mathrm{MRR} = \frac{1}{|F|} \sum_{v \in F} \frac{1}{r_v} \) where
  $F \subset V$ is the set of evaluation nodes, and $r_v$ is the
  highest rank of the gold label among all labels assigned to node
  $v$.}

\subsection{Checking for Label Skew}
\label{sec:label_skew}

Our benchmark tasks are $k$-sparse for small $k$, so a compact sketch
is possible.  However, analysis suggests other scenarios will also be
amenable to small sketches.  While community structure is expensive to
find in large graphs, it is relatively straightforward to check for
skewed label scores.  To this end, we applied \MadExact{} on two
datasets, Freebase and Flickr-10k.  Plots showing mean label scores
vs.~label rank in the log-log space are shown in
\reffig{fig:freebase_lab_score_dist}, along with a plot of
$\frac{1}{rank}$ for each dataset.  In Zipfian distribution with
$z=1$, these plots should be asymptotically parallel.  We indeed see
highly skewed distributions, with most of the weight concentrated on
the largest scores. These observations demonstrate existence of skewed
label scores (during exact label propagation) in real-world datasets,
thereby providing empirical justification for the analysis in
\refsec{sec:analysis}.

\subsection{Is Sketch-based Graph SSL Effective?}
\label{sec:expt_performance}

In this section, we evaluate how \MadSketch{} compares with \MadExact{} in terms of memory usage, runtime, and performance measured using MRR. We set $\mu_1 = 0.98, \mu_2 = 0.01, \mu_3 = 0.01$, and run both algorithms for 10 iterations, which we found to be sufficient for convergence. Experimental results comparing these two algorithms when applied over the Freebase dataset are shown in \reftbl{tbl:mem_time_mrr_compare}. 

First, we compare \MadExact{} and \MadSketch{} $(w = 109, d = 8)$, where the width and depth of the sketch is as suggested by Theorem 2 (see \reftbl{tbl:datasets}). From \reftbl{tbl:mem_time_mrr_compare}, we observe that \MadSketch{} $(w = 109, d = 8)$ achieves the same MRR as \MadExact{}, while using a lower per-iteration memory footprint. Moreover, \MadSketch{} also achieves a significant speedup of $4.7x$ compared to \MadExact{}. This is our main result, which validates the central thesis of the paper: \textit{Sketch-based data structures when used within graph-based SSL algorithms can result in significant memory and runtime improvements, without degrading performance}. This also provides empirical validation of the analysis presented earlier in the paper.

In \reftbl{tbl:mem_time_mrr_compare}, we also explore effect of
sketches of various sizes on the performance of \MadSketch{}. We
observe that even though very small sketches do degrade accuracy,
sketches smaller than predicted by Theorem~\ref{thm:main} do not,
suggesting that the existence of additional useful structure in these
problems---perhaps a combination of label skew and/or community
structure.

\subsection{Graph SSL with Large Number of Labels}
\label{sec:large_labels}

In previous section, we have observed that \MadSketch{}, a
sketch-based graph SSL algorithm, can result in significant memory
savings and runtime speedups compared to \MadExact{}. In this section,
we evaluate sketches in graph SSL when a large number of labels are
involved. For this, we run \MadExact{} and \MadSketch{} over the
Flickr-1m, a dataset with 1 million labels. Both algorithms were run
for 20 iterations, and plots comparing their per-iteration memory
usage are shown in \reffig{fig:mad_flickr_mem_usage}. From this
figure, we observe that even though \MadExact{} starts out with a
lower memory footprint (due to its sparse representation of label
scores, which are very sparse initially), it runs out of memory by the
third iteration, even when 100GB RAM was available. In contrast,
\MadSketch{} when run with a sketch of $w = 55, d = 17$ as prescribed
by Theorem 2 (see \reftbl{tbl:datasets}) is able to compactly store
all labels and their scores, and does not result in an explosion of
space requirements in later iterations. This demonstrates scalability
of \MadSketch{}, and sketch-based Graph SSL in general, when applied
to datasets with large number labels---the main motivation of
this paper.

\section{Conclusion}
\label{sec:conclusion}

Graph-based Semi-supervised learning (SSL) algorithms have been successfully used in a large number of applications. Such algorithms usually require $O(m)$ space on each node. Unfortunately, for many applications of practical significance with very large $m$ over large graphs, this is not sufficient. In this paper, we propose \MadSketch{}, a novel graph-based SSL algorithm which compactly stores label distribution on each node using Count-min Sketch, a randomized data structure. We present theoretical analysis showing that under mild conditions, \MadSketch{} can reduce space complexity at each node from $O(m)$ to $O(\log m)$, and achieve similar savings in time complexity as well. We support our analysis through experiments on multiple real world datasets. We find that \MadSketch{} is able to achieve same performance as \MadExact{}, a state-of-the-art graph-based algorithm with exact estimates, while requiring smaller memory footprint  and at the same time achieving up to 10x speedup. Also, we find that \MadSketch{} is able to scale up to datasets containing millions of nodes and edges, and more importantly one million  labels, which is beyond the scope of existing graph-based SSL algorithms such as \MadExact{}. As part of future work, we hope to explore how such sketching ideas may be effectively used in other graph-based SSL techniques (e.g., \cite{subramanya2011semi}) which use loss functions other than the squared-loss.

\subsubsection*{Acknowledgments}
This work was supported in part by DARPA (award number FA87501320005),  IARPA (award number FA865013C7360), and Google. Any opinions, findings, conclusions and recommendations expressed in this papers are the authors' and do not necessarily reflect those of the sponsors.

\begin{small}
\bibliographystyle{abbrv}
\bibliography{graph_ssl_sketch}

\begin{thebibliography}{10}

\bibitem{agrawal2013multi}
R.~Agrawal, A.~Gupta, Y.~Prabhu, and M.~Varma.
\newblock Multi-label learning with millions of labels: Recommending advertiser
  bid phrases for web pages.
\newblock In {\em Proceedings of the 22nd international conference on World
  Wide Web}, 2013.

\bibitem{baluja2008video}
S.~Baluja, R.~Seth, D.~Sivakumar, Y.~Jing, J.~Yagnik, S.~Kumar,
  D.~Ravichandran, and M.~Aly.
\newblock Video suggestion and discovery for youtube: taking random walks
  through the view graph.
\newblock In {\em Proceedings of WWW}, 2008.

\bibitem{bengio2010label}
S.~Bengio, J.~Weston, and D.~Grangier.
\newblock Label embedding trees for large multi-class tasks.
\newblock {\em NIPS}, 23(163-171):3, 2010.

\bibitem{DBLP:conf/aaai/CarlsonBKSHM10}
A.~Carlson, J.~Betteridge, B.~Kisiel, B.~Settles, E.~R.~H. Jr., and T.~M.
  Mitchell.
\newblock Toward an architecture for never-ending language learning.
\newblock In {\em AAAI}, 2010.

\bibitem{carter1979universal}
J.~L. Carter and M.~N. Wegman.
\newblock Universal classes of hash functions.
\newblock {\em Journal of computer and system sciences}, 18(2):143--154, 1979.

\bibitem{cormode2005improved}
G.~Cormode and S.~Muthukrishnan.
\newblock An improved data stream summary: the count-min sketch and its
  applications.
\newblock {\em Journal of Algorithms}, 55(1):58--75, 2005.

\bibitem{cormode2005summarizing}
G.~Cormode and S.~Muthukrishnan.
\newblock Summarizing and mining skewed data streams.
\newblock In {\em Proc. of ICDM}, pages 44--55, 2005.

\bibitem{das2012graph}
D.~Das and N.~A. Smith.
\newblock Graph-based lexicon expansion with sparsity-inducing penalties.
\newblock In {\em Proceedings of the 2012 Conference of the North American
  Chapter of the Association for Computational Linguistics: Human Language
  Technologies}, pages 677--687. Association for Computational Linguistics,
  2012.

\bibitem{dekel2002multiclass}
O.~Dekel and Y.~Singer.
\newblock Multiclass learning by probabilistic embeddings.
\newblock In {\em In NIPS}, pages 945--952, 2002.

\bibitem{deng2009imagenet}
J.~Deng, W.~Dong, R.~Socher, L.-J. Li, K.~Li, and L.~Fei-Fei.
\newblock Imagenet: A large-scale hierarchical image database.
\newblock In {\em CVPR}. IEEE, 2009.

\bibitem{goyal2010sketch}
A.~Goyal, J.~Jagarlamudi, H.~Daum{\'e}~III, and S.~Venkatasubramanian.
\newblock Sketch techniques for scaling distributional similarity to the web.
\newblock In {\em Proceedings of the 2010 Workshop on GEometrical Models of
  Natural Language Semantics}. Association for Computational Linguistics, 2010.

\bibitem{rao2009ranking}
D.~Rao and D.~Yarowsky.
\newblock {Ranking and Semi-supervised Classification on Large Scale Graphs
  Using Map-Reduce}.
\newblock {\em TextGraphs}, 2009.

\bibitem{sarlos2006randomize}
T.~Sarl{\'o}s, A.~A. Bencz{\'u}r, K.~Csalog{\'a}ny, D.~Fogaras, and
  B.~R{\'a}cz.
\newblock To randomize or not to randomize: space optimal summaries for
  hyperlink analysis.
\newblock In {\em Proceedings of WWW}, 2006.

\bibitem{shi2009hash}
Q.~Shi, J.~Petterson, G.~Dror, J.~Langford, A.~Smola, and S.~Vishwanathan.
\newblock Hash kernels for structured data.
\newblock {\em The Journal of Machine Learning Research}, 10:2615--2637, 2009.

\bibitem{song2009scalable}
H.~H. Song, T.~W. Cho, V.~Dave, Y.~Zhang, and L.~Qiu.
\newblock Scalable proximity estimation and link prediction in online social
  networks.
\newblock In {\em Proceedings of ACM SIGCOMM}, 2009.

\bibitem{subramanya2011semi}
A.~Subramanya and J.~Bilmes.
\newblock Semi-supervised learning with measure propagation.
\newblock {\em The Journal of Machine Learning Research}, 12:3311--3370, 2011.

\bibitem{talukdar:nra09}
P.~P. Talukdar and K.~Crammer.
\newblock New regularized algorithms for transductive learning.
\newblock In {\em ECML-PKDD}, 2009.

\bibitem{talukdar2010experiments}
P.~P. Talukdar and F.~Pereira.
\newblock Experiments in graph-based semi-supervised learning methods for
  class-instance acquisition.
\newblock In {\em Proceedings of ACL}, 2010.

\bibitem{talukdar2008wsa}
P.~P. Talukdar, J.~Reisinger, M.~Pasca, D.~Ravichandran, R.~Bhagat, and
  F.~Pereira.
\newblock {Weakly-Supervised Acquisition of Labeled Class Instances using Graph
  Random Walks}.
\newblock In {\em Proceedings of EMNLP}, 2008.

\bibitem{Weinberger:2009:FHL:1553374.1553516}
K.~Weinberger, A.~Dasgupta, J.~Langford, A.~Smola, and J.~Attenberg.
\newblock Feature hashing for large scale multitask learning.
\newblock In {\em Proceedings of ICML}, 2009.

\bibitem{wu2012learning}
X.-M. Wu, Z.~Li, A.~M.-C. So, J.~Wright, and S.-F. Chang.
\newblock Learning with partially absorbing random walks.
\newblock In {\em NIPS}, 2012.

\bibitem{YangChute94}
Y.~Yang and C.~Chute.
\newblock An example-based mapping method for text classification and
  retrieval.
\newblock {\em ACM Transactions on Information Systems}, 12(3), 1994.

\bibitem{zhu2003semi}
X.~Zhu, Z.~Ghahramani, and J.~Lafferty.
\newblock Semi-supervised learning using gaussian fields and harmonic
  functions.
\newblock In {\em ICML}, 2003.

\end{thebibliography}
\end{small}

\end{document}